\documentclass[letterpaper,10pt,english]{IEEEtran}
\usepackage{amsfonts}       
\usepackage{hyperref}
\usepackage{algorithm}
\usepackage{algorithmicx}
\usepackage{amssymb}
\usepackage{amsmath}
\usepackage{graphicx,wrapfig,times,amsthm,bm}
\usepackage{subfig}
\usepackage{threeparttable}
\usepackage[wide]{sidecap}
\usepackage{array}
\usepackage{algorithm}
\usepackage{algpseudocode}
\usepackage{lipsum,babel}
\usepackage{color}
\usepackage{setspace}

\newtheorem{thm}{Theorem}
\newtheorem{lem}{Lemma}
\newtheorem{defn}{Definition}

\newtheorem{assum}{Assumption}

\newtheorem{rem}{Remark}
\newtheorem{prob}{Problem}

\def \R {\mathbb{R}}

\def \tailcoeff {\gamma}
\def \tailexp {\alpha}

\def \ssconstant {C}

\def \policy {\boldsymbol{\pi}}

\def \optimalpolicy {\boldsymbol{\pi^\star}}
\def \para {\boldsymbol{\theta}}
\def \rrate {\beta}

\def \randommatrix {\boldsymbol{\phi}}
\def \bigrandommat {\boldsymbol{\psi}}

\newcommand{\Mnorm}[2]{{\left\vert\kern-0.30ex\left\vert\kern-0.30ex\left\vert #1 
		\right\vert\kern-0.30ex\right\vert\kern-0.30ex\right\vert}}
\newcommand{\Opnorm}[3]{{\left\vert\kern-0.25ex\left\vert\kern-0.25ex\left\vert #1 
		\right\vert\kern-0.25ex\right\vert\kern-0.25ex\right\vert}_{#2 \to #3}}
\newcommand{\norm}[2]{{\left\vert\kern-0.30ex\left\vert #1 
		\right\vert\kern-0.30ex\right\vert}}
\newcommand{\eigmax}[1]{\left| \boldsymbol{\lambda}_{\max} \left( #1 \right)\right|}
\newcommand{\eigmin}[1]{\left| \boldsymbol{\lambda}_{\min} \left( #1 \right)\right|}
\newcommand{\tr}[1]{\boldsymbol{\mathrm{tr}} \left( #1 \right)}
\newcommand{\PP}[1]{\mathbb{P} \left(#1\right)}
\newcommand{\E}[1]{\mathbb{E} \left[#1\right]}

\newcommand{\innerproductminconstant}[1]{\psi_0}

\newcommand{\regret}[2]{\mathcal{R}_{#1} \left(#2\right)}

\newcommand{\Kmatrix}[1]{P\left(#1\right)}
\newcommand{\Lmatrix}[1]{G\left(#1\right)}
\newcommand{\extendedLmatrix}[1]{F\left(#1\right)}
\newcommand{\extendedL}[1]{F_{#1}}

\newcommand{\instantcost}[2]{c_{#1}\left(#2\right)}

\newcommand{\estpara}[1]{\widetilde{\boldsymbol{\theta}}_{#1}}

\newcommand{\levelset}[1]{\mathcal{S}\left({#1}\right)}

\newcommand{\riccatiOp}[2]{\Phi_{#1}\left(#2\right)}
\newcommand{\covmat}[1]{\Sigma_{#1}}
\newcommand{\empiricalcovmat}[1]{\widehat{\Sigma}_{#1}}

\newcommand{\boot}[1]{\widehat{#1}}
\newcommand{\Pboot}[1]{\widehat{\mathbb{P}}_{#1}}
\newcommand{\Eboot}[1]{\widehat{\mathbb{E}}_{#1}}
\newcommand{\dirac}[2]{\boldsymbol{\delta} {\left[#1\right]}}
\newcommand{\history}[1]{\mathbb{H}_{#1}}
\newcommand{\model}[1]{\mathcal{#1}}

\newcommand{\filter}[1]{\left\{\mathcal{F}_{#1}\right\}_{{#1}=0}^\infty}
\newcommand{\noise}[1]{\xi\left(#1\right)}
\newcommand{\state}[1]{x\left(#1\right)}
\newcommand{\action}[1]{u\left(#1\right)}
\newcommand{\bootstate}[1]{\widehat{x}\left(#1\right)}
\newcommand{\residunoise}[1]{\zeta\left(#1\right)}
\newcommand{\bootnoise}[1]{\widehat{\xi}\left(#1\right)}
\newcommand{\avenoise}[1]{\overline{\zeta}_{#1}}

\newif\ifarxiv
\arxivtrue

\begin{document}
\ifarxiv
\doublespacing
\onecolumn
\else \fi
\title{On Applications of Bootstrap in Continuous Space Reinforcement Learning}
\author{Mohamad~Kazem~Shirani~Faradonbeh,
	Ambuj~Tewari,
	and~George~Michailidis
	\thanks{}}
\maketitle

\begin{abstract}
	In decision making problems for continuous state and action spaces, linear dynamical models are widely employed. Specifically, policies for stochastic linear systems subject to quadratic cost functions capture a large number of applications in reinforcement learning. Selected randomized policies have been studied in the literature recently that address the trade-off between identification and control. However, little is known about policies based on bootstrapping observed states and actions. In this work, we show that bootstrap-based policies achieve a square root scaling of regret with respect to time. We also obtain results on the accuracy of learning the model's dynamics. Corroborative numerical analysis that illustrates the technical results is also provided.
\end{abstract}

\section{Introduction} \label{intro}
In the theory of reinforcement learning, efficient algorithms with provable theoretical guarantees are established for two canonical settings. The first is \emph{finite state} Markov decision processes (MDPs) with state spaces of small cardinalities \cite{li2012sample}. The second is the continuous space setting of \emph{linear quadratic (LQ)} systems \cite{dorato1995linear}. In the latter one, the control action and the state both are multidimensional real vectors, and the state evolves according to stochastic linear dynamics determined by the control action. Further, the cost (or negative reward) has a quadratic form in both the state and the control input. Besides being theoretically amenable, LQ models capture a wide range of applications from air conditioning control \cite{lazic2018data} to portfolio optimization \cite{abeille2016lqg}. LQ models also arise when studying the behavior of nonlinear systems around the working equilibrium~\cite{li2004iterative,kappen2005linear}

In applications where the true system model is not known, data-driven strategies are required for decision making under uncertainty \cite{lai1986asymptotically}. Then, the learning algorithm has to select actions amongst infinitely many options in order to steer the system toward minimizing the costs incurred. Note that, unlike the finite state MDP case, in LQ systems there is a possible danger of the state vector becoming unbounded \cite{lai1985asymptotic,faradonbeh2018finite,sarkar2018fast}. Therefore, the design and analysis of reinforcement learning algorithms for LQ systems involve significantly different conceptual and technical issues to balance exploration (identification) and exploitation (control). For this purpose, one might consider to use upper-confidence bound (UCB) approaches~\cite{campi1998adaptive,bittanti2006adaptive,abbasi2011regret,faradonbeh2017finite} that rely on the optimism in the face of uncertainty~(OFU) principle. The UCB approach was historically first developed for finite action bandit problems \cite{lai1985asymptotically}. While being efficient in the finite action setting, UCB-based approaches have been found to be computationally intractable in more general problems \cite{faradonbeh2018input}. 

Recently, various methods for reinforcement learning have been proposed that leverage \emph{randomization} strategies to guide the learning process. Randomized policy search methods have been studied both empirically, as well as theoretically (see e.g. \cite{mania2018simple,malik2018derivative}). For the problem of stabilizing an unknown LQ system, an algorithm leveraging random feedback gains is proposed~\cite{faradonbeh2018stabilization}. There is also work showing the efficiency of achieving the exploration-exploitation trade-off by randomizing the learned model through both posterior sampling~\cite{abeille2018improved} and additive randomization~\cite{faradonbeh2018optimality}. Finally, finite time analysis of Certainty Equivalent policies utilizing input perturbation has led to performance guarantees for both learning~\cite{dean2018safely} and planning~\cite{faradonbeh2018input}. 

In this paper, we study randomized algorithms that leverage the \emph{statistical bootstrap}~\cite{efron1979bootstrap} for reinforcement learning in LQ systems. Bootstrap-based exploration has been analyzed in simpler settings, such as bandit problems~\cite{kveton2018garbage,vaswani2018new}. There has been a lot of interest recently in using bootstrap-based exploration strategies especially along with deep neural networks~\cite{eckles2014thompson,osband2015bootstrapped,osband2016deep,rajagopal2017neural}. However, results on bootstrap-based reinforcement learning algorithms for LQ models have been limited to primarily numerical analyses for learning the model-misspecification error \cite{dean2017sample}, while {\em rigorous performance guarantees} are not currently available.

Further, bootstrap methods are also of practical interest because of their {\em robustness} to misspecified models. The amount of exploration in bootstrap-based adaptive control policies is endogenously determined by the history of the system to date. Therefore, the policy adapts its decision-making strategy with possible systematic and/or latent ``biases" occurring due to lack of accurate information regarding the system's dynamics\footnote{see the discussion at the end of Section \ref{analysis} for more details}. Examples of such ``biases" include structural breaks~\cite{pesaran2005small}, system resets~\cite{dean2017sample}, and misspecification of the model dimension~\cite{byrnes1994nonlinear,guo1996global,todorov2005generalized}. 

The focus of this work is on the performance of reinforcement learning policies that use the \emph{residual bootstrap} to balance exploration and exploitation. We show that model-based strategies that use linear regression for learning the model and bootstrapping for policy design, provide a regret that scales as the square root of the total time of interacting with the system. Further, the accuracy of learning the unknown dynamics parameter will be specified. To establish the results, we carefully examine the effect of different converging and diverging quantities involved in the problem, such as the errors in learning the model, the distributions induced by the regression residuals, the correlation within and between the observed input-state signals, and the ongoing learning-planning interactions, that are of independent interest. At the technical level, we leverage results in the literature on the bootstrap~\cite{hall2013bootstrap}, martingale central limit~\cite{brown1971martingale,mcleish1974dependent} and convergence theorems~\cite{hall2014martingale}.

The remainder of the paper is organized as follows: Section \ref{model} introduces the mathematical model under consideration and discusses the rigorous formulation of the problem, and also provides some necessary preliminaries. Section \ref{algos} describes the bootstrap procedure and the resulting reinforcement learning algorithm to design the policy. Subsequently, the main result on the performance of the proposed algorithm is presented, together with numerical work showcasing the performance of the algorithm, in Section \ref{analysis}. 

\paragraph{Notation} 
The following notation will be employed throughout the paper. $A'$ is the transpose of matrix or vector $A$. The largest and the smallest eigenvalues of square Hermitian matrix $A$ are denoted by $\lambda_{\max} (A)$ and $\lambda_{\min}(A)$, respectively. If $A$ is not Hermitian, the ordering of the eigenvalues is determined by their magnitudes. 
The norm of the $d$ dimensional vector $v$ is denoted by $\norm{v}{} = \left( \sum\limits_{i=1}^{d} \left| v_i \right|^2 \right)^{1/2}$, and $\Mnorm{\cdot}{}$ is used for the operator norm of matrices: $\Mnorm{A}{2} = \sup \limits_{\norm{v}{2}=1} \norm{Av}{2}$.
For atomic measures on Euclidean spaces we use Dirac function $\dirac{v}{\cdot}$; i.e. it denotes a unit point mass at $v \in \R^d$. Finally, the letters $\policy$ and $\para$ (or $\estpara{},\boot{\para}$) are being used for generic reinforcement learning policies and model parameters, respectively, and will be rigorously defined later on. 

\section{Setting and Problem Formulation} \label{model}
The model, denoted by $\model{M}$, consists of multidimensional state and control vectors, parameters that specify its dynamical evolution over time and cost matrices, as defined next. The $p$ dimensional state process $\left\{ \state{t} \right\}_{t=0}^\infty$ evolves according to an {\em unknown} stochastic linear dynamic equation governed by the $r$ dimensional control action $\action{t}$, and the random disturbance (or noise) process~$\left\{ \noise{t} \right\}_{t=1}^\infty$:
\begin{eqnarray}
\state{t+1} &=& A_0\state{t}+B_0\action{t}+ \noise{t+1}. \label{systemeq1}
\end{eqnarray}
That is, the current state $\state{t}$ and the input $\action{t}$ determine the next state $\state{t+1}$ through the state transition matrix $A_0 \in \R^{p \times p}$, and the input influence matrix $B_0 \in \R^{p \times r}$, respectively. 
\begin{defn}
	Henceforth, we will denote the true parameter tuple $\left[A_0,B_0\right] \in \R^{p \times (p+r)}$ by the $p \times q$ dynamics matrix $\para_0$, with $q \equiv p+r$. Similarly, we use the parameter $\para \in \R^{p \times q}$ to denote generic dynamics matrices.
\end{defn}

The additive noise in the stochastic dynamics \eqref{systemeq1} satisfies $\E{\noise{t}}=0$. For the sake of simplicity, we assume that the sequence of noise vectors are independent, and have a stationary covariance structure: $\E{\noise{t}\noise{t}'}=\covmat{}$. Further, $\covmat{}$ is assumed to be positive definite, and $\sup\limits_{t \geq 1} \E{\norm{\noise{t}}{2}^{\tailexp}}<\infty$, for some fixed $\tailexp>4$. As a matter of fact, extensions to more general technical settings such as non-stationary~\cite{faradonbeh2018input} or singular covariance matrices (assuming reachability \cite{faradonbeh2018finite}), as well as conditionally independent processes~\cite{abbasi2011regret}, can be accommodated in a similar manner. Note though that the assumed noise
process is not necessarily stationary in the strict sense.

We are interested in finding reinforcement learning policies to minimize the long-term average cost as formally defined next. First, suppose that $Q_x$ and $Q_u$ are the regulation weight matrices reflecting the effect of the state and the input vectors in the cost function, respectively. Specifically, letting $\policy$ be the decision making law (policy) determining the control input $\action{t}$ at every time $t$, define the quadratic instantaneous cost of $\policy$ according to
\begin{eqnarray}
\instantcost{t}{\policy}&=& \norm{Q_x^{1/2}\state{t}}{2}^2 + \norm{Q_u^{1/2}\action{t}}{2}^2, \label{systemeq2}
\end{eqnarray}
where $Q_x \in \R^{p \times p}$, and $Q_u \in \R^{r \times r}$ are symmetric positive definite matrices. Thus, \eqref{systemeq2} reflects the desire to regulate the state of the system through control actions of small magnitude.

When the dynamics follow \eqref{systemeq1}, and the instantaneous cost is given by \eqref{systemeq2}, we denote the model by~$\model{M}\left(\para_0\right)= \left( \para_0, Q_x,Q_u\right)$. Further, the history of the system at time $t$, denoted as $\history{t}$, consists of the sequence of the control inputs applied so far, and the resulting state vectors:
\begin{equation*}
\history{t}=\left( \state{0}, \cdots, \state{t}, \action{0}, \cdots, \action{t-1} \right).
\end{equation*}
A reinforcement learning policy observes the history $\history{t}$ at time $t$ aiming to control the cost incurred.
That is, the policy $\policy$ is a (possibly random) mapping which designs the input sequence~$\left\{\action{t}\right\}_{t=0}^\infty$ according to the history available up to that time;
\begin{equation} \label{systemeq3}
\action{t} = \policy \left( \history{t},Q_x,Q_u \right),
\end{equation}
so that the average cost is minimized. Thus, the objective is summarized in the following regulation problem:
\begin{prob} \label{exploitP}
	Find $\policy$ to minimize the average cost below, subject to \eqref{systemeq1}, \eqref{systemeq2}, and \eqref{systemeq3};
	\begin{equation} \label{EAcostdeff}
		\limsup \limits_{n \to \infty} {\frac{1}{n} \sum\limits_{t=0}^{n-1} \instantcost{t}{\policy}}.
	\end{equation}
\end{prob} 
Importantly, according to \eqref{systemeq3} the true dynamics parameter $\para_0$ in \eqref{systemeq1} is unknown. Therefore, the policy must also employ an exploration procedure to accurately learn the model parameters, thus addressing the following identification problem:
\begin{prob} \label{exploreP}
	Using \eqref{systemeq1} and \eqref{systemeq3}, design $\policy$ to learn $\para_0$, as accurately as possible. 
\end{prob}
Note that in the above formulation, the true system dynamics are unknown, while the cost matrices are known. It gives rise to  a realistic setting, since the decision making algorithm does not know the actual evolution of the underlying system (i.e. $\para_0$), but is aware of the criteria according to which a policy to achieve the goal is being assessed (i.e. $Q_x,Q_u$).

Subsequently, we define the {\em regret} of a policy, which is the amount of sub-optimality it incurs due to uncertainty about the parameters of the model \eqref{systemeq1}. To do so, we need to introduce an optimal policy $\optimalpolicy$ that minimizes the average cost, given full knowledge of the system model $\model{M}\left(\para_0\right)$. Then, $\optimalpolicy$ will be the baseline for assessing the exploitation performance of the arbitrary reinforcement learning policy $\policy$. It is well known that in order to find $\optimalpolicy$, an algebraic Riccati equation needs to be solved~\cite{bertsekas1995dynamic,kumar2015stochastic}. 

To proceed, we introduce some additional notation. First, for an arbitrary $\para=\left[ A,B \right]$ define the matrix valued mapping 
\begin{eqnarray*} \label{ricatti2}
	\riccatiOp{\para}{P} = Q_x + A'PA - A' P B \left(B'PB+Q_u\right)^{-1} B'PA.
\end{eqnarray*}
Both the domain and the range of $\riccatiOp{\para}{\cdot}$ are the set of $p \times p$ matrices. Next, if there is a positive semidefinite matrix $\Kmatrix{\para}$ satisfying the algebraic Riccati equation $\Kmatrix{\para}=\riccatiOp{\para}{\Kmatrix{\para}}$, let the feedback gain matrix $\Lmatrix{\para}$ be
\begin{eqnarray} \label{ricatti1}
\Lmatrix{\para} = -\left(B'\Kmatrix{\para}B+ Q_u \right)^{-1} B'\Kmatrix{\para}A.
\end{eqnarray}
Furthermore, for $\para_0=[A_0,B_0]$ in \eqref{systemeq1}, define the linear time-invariant (LTI) policy 
\begin{eqnarray} \label{optpolicydeff}
\optimalpolicy: \:\:\: \action{t}= \Lmatrix{\para_0} \state{t}, \:\:\: t=0,1,2,\cdots.
\end{eqnarray}
Finally, using $\optimalpolicy$, the regret of $\policy$ is naturally defined by: 
\begin{eqnarray} \label{regretdeff}
\regret{n}{\policy} = \sum \limits_{t=0}^{n-1} \left[ \instantcost{t}{\policy} - \instantcost{t}{\optimalpolicy} \right].
\end{eqnarray}

It remains to specify settings for which $\optimalpolicy$ is well-defined. To that end, the following closed-loop stabilizability condition for the model \eqref{systemeq1} is necessary and sufficient \cite{bertsekas1995dynamic,kumar2015stochastic}.
\begin{assum} \label{stabilizability}
	There is a LTI feedback gain $\action{t}=G_s\state{t}$, such that $G_s \in \R^{r \times p}$ satisfies $\eigmax{A_0+B_0G_s} < 1$. 
\end{assum} 
Note that in general, the stabilizing gain $G_s$ mentioned above is only required to exist, and does not need to be known to the decision maker. In other words, to verify that the stabilizability Assumption \ref{stabilizability} holds, it suffices to show that a hypothetical omniscient decision maker (who knows the true model $\model{M}\left(\para_0\right)$) possessing an omnipotent computational power is able to stabilize the system. However, we briefly outline the available constructive methods to compute $G_s$ (as well as $\optimalpolicy$). It is shown (see for example \cite{bertsekas1995dynamic,kumar2015stochastic,faradonbeh2018stabilization}) that under Assumption \ref{stabilizability} the following statements hold;
\begin{enumerate}
	\item 
	The positive definite matrix $\Kmatrix{\para_0}$ uniquely exists. So, both the feedback $\Lmatrix{\para_0}$ and the optimal policy $\optimalpolicy$ are well defined.
	\item 
	Letting $P_0$ be an arbitrary positive semidefinite $p \times p$ matrix, the recursive formula $P_{k+1}=\riccatiOp{\para_0}{P_k}$ converges exponentially fast to $\Kmatrix{\para_0}$ as $k$ grows.
	\item 
	The feedback matrix $\Lmatrix{\para_0}$ stabilizes the system: 
	$$\eigmax{A_0+B_0 \Lmatrix{\para_0}}<1.$$
	\item 
	The minimum of the average cost \eqref{EAcostdeff} is achieved by $\optimalpolicy$.
	\item 
	In the class of LTI policies (i.e. of the form $\action{t}=G\state{t}$), the policy $\optimalpolicy$ is the only optimal one. 
\end{enumerate}

In the remainder of the paper, we employ reinforcement learning algorithms to address Problem \ref{exploitP}, studying the growth rates of $\regret{n}{\policy}$. Similarly, letting $\estpara{n}$ be the learned/estimated parameter at time $n$ (the sample size is $n$ as well), we consider the exploration performance in Problem \ref{exploreP} through the rates of the learning error $\Mnorm{\estpara{n}-\para_0}{}$. Bootstrap is the cornerstone of the proposed algorithms to efficiently randomize the design of the control inputs, and address the trade-off between the learning accuracy and the regret. 

\section{Algorithms} \label{algos}
An algorithm needs to address the common dilemma of decision making under uncertainty, as follows. First, if the algorithm makes decisions naively according to the estimated (learned) dynamics parameter, it will presumably fail to provide a small regret. Intuitively, the state $\state{t}$ and the action $\action{t}$ are required to be highly correlated in order to remain close to the optimal strategy $\optimalpolicy$ in \eqref{optpolicydeff}. Because of this correlation,  history $\history{t}$ may fail to accurately learn $\para_0$, which can lead to drastically large regret values. Technically, if $\action{t}=G\state{t}$ for some $r \times p$ feedback gain matrix $G$, then the dimension of the observed history is effectively $p$, while the rows of the matrices $\para$ in the parameter space belong to $\R^q$. Therefore, learning can be dramatically misleading. This phenomenon of {\em failing to falsify} the imprecise approximations of the true model is extensively discussed in the adaptive control literature \cite{becker1985adaptive,lai1986extended,bittanti2006adaptive,faradonbeh2018optimality}. 

In other words, if the policy fails to sufficiently explore the parameter space, an \emph{inaccurate} approximation $\estpara{t}$ can falsely be treated as an accurate one. This necessitates an efficient exploration strategy to decrease the aforementioned correlation between the state and the action. Moreover, the above argument reveals the reasoning leading to UCB approaches~\cite{abbasi2011regret,faradonbeh2017finite}, or statistically independent dither schemes~\cite{faradonbeh2018input,dean2018safely}, as useful prescriptions to overcome the exploration-exploitation dilemma.

In order to explore, the decision maker needs to deviate from the learned model $\estpara{t}$ prior to using $\model{M}\left( \estpara{t}\right)$ to design the reinforcement learning policy. On the other hand though, the above deviations must be sufficiently small in order to avoid significant deterioration in the exploitation performance (i.e. increase in the regret). The solution we discuss here is to utilize the bootstrap to provide the necessary balance between these two competing objectives.

To this end, the policy $\policy$ applies the supposedly optimal control action treating $\model{M}\left(\boot{\para}_t\right)$ as the true model, where $\boot{\para}_t$ is provided by the bootstrap algorithm. It computes the regression residuals for the learned parameter $\estpara{t}$, and bootstraps (i.e. resamples) them to reconstruct a {\em surrogate} system. Then, the history of the surrogate system will be the data being used to compute~$\boot{\para}_t$. In the first subsection, we explain the least squares estimator for learning the model parameter, as well as the above residual bootstrap procedure.

Subsequently, in the second subsection we present an episodic algorithm which updates the model-based policy at the end of every episode, while the lengths of the episodes grow exponentially fast. Therefore, as the duration of the interaction with the system grows, the number of policy updates scales logarithmically. As a matter of fact, this leads to a significant reduction in the computation of the reinforcement learning policy, by avoiding unnecessary updates before collecting sufficient data, due to the fact that the solution of the algebraic Riccati equation \eqref{ricatti1} for a hypothetical model is not instantly available. The latter would impose a substantial computational burden, especially for systems whose dimension is fairly large.

\subsection{Residual Bootstrap}
According to the linear dynamical model in \eqref{systemeq1}, a natural procedure to learn $\para_0$ through the control input $\action{t}$ and the observed states $\state{t},\state{t+1}$ is based on least squares. In the sequel, we discuss the residual bootstrap method for the least squares learning procedure. Further, we will present the corresponding algorithm which will be used as a subroutine in the reinforcement learning algorithm in the next subsection. 

Recall that the LTI policy $\optimalpolicy$ in \eqref{optpolicydeff} is optimal. Thus, a natural form of the adaptive policies that a reinforcement learning algorithm is expected to provide through planing according to the learned model, is $\action{t}=G_t\state{t}$. Assuming so for $t<n$, now the algorithm needs to decide about the action at time $n$. Thus, plugging $\action{t}=G_t\state{t}$ in the dynamical model \eqref{systemeq1}, and denoting
\begin{equation*}
\extendedL{t}=\left[ I_p, G_t' \right]' \in \R^{q \times p},
\end{equation*}
we get the so-called closed-loop evolution of the system by the (possibly time-varying) autoregressive dynamics
\begin{equation*}
\state{t+1}= \para_0 \extendedL{t} \state{t} + \noise{t+1},
\end{equation*}
for $0 \leq t <n$. Then, Algorithm \ref{bootalgo} returns the bootstrapped parameter $\boot{\para}_n$ based on the matrices $\left\{ \extendedL{t} \right\}_{t=0}^{n-1}$, as well as the available state observations $\left\{ \state{t} \right\}_{t=0}^n$. The details are provided below.

First, based on the collected history $\left\{ \extendedL{t} \right\}_{t=0}^{n-1}, \left\{ \state{t} \right\}_{t=0}^n$, define the following least square estimate of $\para_0$:
\begin{eqnarray} \label{LSE}
	\estpara{n} = \arg\min\limits_{\para \in \R^{p \times q}} \sum\limits_{t=0}^{n-1} \norm{\state{t+1}- \para \extendedL{t} \state{t}}{}^2.
\end{eqnarray}
The learning procedure \eqref{LSE} treats the noise vectors $\noise{t}$ as the errors of a linear regression procedure, based on the dynamical model \eqref{systemeq1}. Therefore, the residuals of the least squares estimate are defined by the difference between the observed response $\state{t+1}$, and the fitted response $\estpara{n} \extendedL{t} \state{t}$. That is,
\begin{equation} \label{residuals}
	\residunoise{t+1}= \state{t+1}- \estpara{n} \extendedL{t} \state{t},
\end{equation}
for $0 \leq t <n$. The residuals $\left\{ \residunoise{t} \right\}_{t=1}^n$ can conceptually be considered as approximations of the actual regression errors $\left\{\noise{t}\right\}_{t=1}^n$. Using the residuals $\left\{ \residunoise{t} \right\}_{t=1}^n$, we define the centered empirical distribution 
\begin{equation} \label{empiricaldist} 
	\Pboot{n}= \frac{1}{n} \sum\limits_{t=1}^n \dirac{\residunoise{t}-\avenoise{n}}{\cdot},
\end{equation}
where $\avenoise{n}$, the average of the residuals given by
\begin{equation} \label{residucenter}
	\avenoise{n}= \frac{1}{n} \sum\limits_{t=1}^n \residunoise{t},
\end{equation}
is being used for centering the empirical distribution. In fact, $\Pboot{n}$ is the sample probability measure for the population distribution of the noise process $\left\{ \noise{t} \right\}_{t=1}^\infty$. Note that $\Pboot{n}$ is defined on $\R^p$. We then use $\estpara{n}$ and $\Pboot{n}$ to generate the surrogate state vectors $\left\{ \bootstate{t} \right\}_{t=0}^n$ by the dynamical model
\begin{equation*}
\bootstate{t+1}= \estpara{n} \extendedL{t}  \bootstate{t} + \bootnoise{t+1},
\end{equation*} 
where the bootstrap noise vectors $\bootnoise{t+1}$ are drawn independently from $\Pboot{n}$. Hence, letting $\Eboot{n}$ be the expectation with respect to $\Pboot{n}$, clearly we have $\Eboot{n}\left[ \bootnoise{t} \right]=0$. Also note that the actual dynamics parameter for the surrogate system $\left\{ \bootstate{t} \right\}_{t=0}^n$ is the learned parameter $\estpara{n}$ defined in \eqref{LSE}. Finally, the algorithm applies the least squares estimator to the generated surrogate states to obtain $\boot{\para}_n$:
\begin{eqnarray} \label{BootLSE}
\boot{\para}_n = \arg\min\limits_{\para \in \R^{p \times q}} \sum\limits_{t=0}^{n-1} \norm{ \bootstate{t+1}- \para \extendedL{t} \bootstate{t}}{}^2.
\end{eqnarray}
The pseudo-code for the residual Bootstrap explained above is given in Algorithm \ref{bootalgo}. It will be used later at the heart of Algorithm \ref{adaptivealgo1} to design reinforcement learning policies.
\begin{algorithm}
	\caption{{: BOOTSTRAP} } \label{bootalgo}
	\begin{algorithmic}
		\State {\bf Inputs:} data $\left\{\state{t}\right\}_{t=0}^n, \left\{\extendedL{t}\right\}_{t=0}^{n-1}$
		\State {\bf Output:} bootstrapped estimate $\boot{\para}_n$
		\State Define $\estpara{n}, \left\{\residunoise{t}\right\}_{t=1}^n, \avenoise{n}$, and $\Pboot{n}$ according to \eqref{LSE}, \eqref{residuals}, \eqref{residucenter}, and \eqref{empiricaldist}, respectively
		\State Let $\bootstate{0}=\state{0}$
		\For{$t=0,1,2,\cdots,n-1$}
		\State Draw $\bootnoise{t+1}$ from $\Pboot{n}$, independently
		\State Let $\bootstate{t+1}= \estpara{n} \extendedL{t}  \bootstate{t} + \bootnoise{t+1}$
		\EndFor
		\State Return $\boot{\para}_n$ given by \eqref{BootLSE}
	\end{algorithmic}
\end{algorithm}

\begin{rem}
	If the noise process is parametrized, one can accordingly draw $\bootnoise{t}$ from the corresponding parametric sample distribution. 
\end{rem}
To see that, assume we know that the noise vectors belong to a parametric family of stochastic processes. Then, instead of using the nonparametric empirical distribution in \eqref{empiricaldist}, one can use the the residuals $\residunoise{t}$ to estimate the parameter of interest. So, letting $\Pboot{n}$ be the parametric distribution provided by the obtained estimate, the bootstrap noise $\bootnoise{t}$ can be sampled independently from $\Pboot{n}$. For example, if we know that $\noise{t}$ are i.i.d. Gaussian vectors, we can find the sample covariance matrix $\empiricalcovmat{n}=n^{-1} \sum\limits_{t=1}^n \residunoise{t}\residunoise{t}'- \avenoise{n}\:\avenoise{n}'$, and draw $\bootnoise{t}$ independently from the centered Gaussian distribution with covariance matrix $\empiricalcovmat{n}$.
\begin{rem}
	In the original version of bootstrap~\cite{efron1979bootstrap}, the covariates (i.e. the state vectors) are fixed, and only the residuals are being bootstrapped. In the time series models such as \eqref{systemeq1}, every state vector comprises of the previous noise vectors. Therefore, bootstrapping the residuals automatically leads to new state sequence $\left\{\bootstate{t}\right\}_{t=0}^\infty$ for the surrogate system~\cite{dean2017sample}.
\end{rem}    

\subsection{Policy Design}
Next, Algorithm \ref{adaptivealgo1} for decision making under uncertainty based on bootstrapping the residuals (Algorithm \ref{bootalgo}) is discussed. For this purpose, we first define the extended gain matrix $\extendedLmatrix{\para}$ based on the optimal feedback $\Lmatrix{\para}$.
\begin{defn} \label{extendedLdeff}
	For parameter $\para \in \R^{p \times q}$, using the matrix $\Lmatrix{\para}$ in \eqref{ricatti1}, define the $q \times p$ matrix $\extendedLmatrix{\para}= \left[ I_p, \Lmatrix{\para}' \right]'$.
\end{defn}
The matrix $\extendedLmatrix{\para}$ can be interpreted as an extension of the original feedback gain; applying $\action{t}=\Lmatrix{\para}\state{t}$, the closed-loop transition matrix takes the form~$A_0+B_0\Lmatrix{\para}=\para_0\extendedLmatrix{\para}$.

Recall that the true model is not known, and a reinforcement learning algorithm needs to simultaneously learn the dynamics parameter, and design the control input. To do so, we present an episodic decision making strategy outlined in Algorithm \ref{adaptivealgo1}. That is, the policy applies control actions during each episode, assuming that the approximation of the model available at the time \emph{coincides} with the true model. Then, at the end of every episode, the algorithm updates the learned model based on the history collected so far, and continues making decisions as if the new approximate model is the truth. The learning mentioned above is through a linear regression for the dynamics \eqref{systemeq1}, and the approximation consists of bootstrapping (by Algorithm \ref{bootalgo}) the model estimate obtained by the regression. In the sequel, we explain the details of the above alternating steps of the algorithm.
\begin{algorithm}
	\caption{{: POLICY DESIGN} } \label{adaptivealgo1}
	\begin{algorithmic}
		\State Let $\history{0}=\left\{ \state{0} \right\}$
		\State Choose stabilizable $\boot{\para}_0$ arbitrarily
		\For{$m=1,2,\cdots$}
		\While{$t < \rrate^{m}$}
		\State Apply feedback gain $\action{t}=\Lmatrix{\boot{\para}_t} \state{t}$ 
		\State Update history $\history{t+1}=\history{t} \cup \left\{ \state{t+1}, \extendedLmatrix{\boot{\para}_t} \right\}$
		\State $\boot{\para}_{t+1}=\boot{\para}_{t}$
		\EndWhile
		\State Update parameter $\boot{\para}_{t+1}=\text{BOOTSTRAP}\left(\history{t+1}\right)$
		\EndFor
	\end{algorithmic}
\end{algorithm}

The reinforcement learning policy is initiated with the history $\history{0}$ in the first line of Algorithm \ref{adaptivealgo1}. Then, it chooses an arbitrary stabilizable approximation of $\para_0$, denoted by $\boot{\para}_0$, and starts the system by applying the action prescribed by the model $\model{M}\left(\boot{\para}_0\right)$; i.e. $\action{t}=\Lmatrix{\boot{\para}_0}\state{t}$. Note that selection of $\boot{\para}_0$ is straightforward, since almost all (w.r.t. Lebesgue measure) parameter matrices are stabilizable \cite{abeille2018improved}. 

The starting time-points of the episodes are determined by the exponents of the reinforcement rate $\rrate>1$. That is, at every time $t = \lceil \rrate^{m}\rceil$, the approximation $\boot{\para}_t$ will be updated, while for $\rrate^{m} \leq t < \rrate^{m+1}$ the algorithm freezes $\boot{\para}_t$. In other words, whenever $t = \lceil\rrate^{m}\rceil$, Algorithm \ref{adaptivealgo1} calls the residual bootstrap Algorithm \ref{bootalgo} to get $\boot{\para}_t$ according to the collected history of the control actions and the states. So, for all $\rrate^{m} \leq t < \rrate^{m+1}$, the matrices $\extendedLmatrix{\boot{\para}_t}$ are exactly the same. The efficiency of the policy relies on the idea that the sequence $\left\{ \boot{\para}_t \right\}_{t=0}^\infty$ will provide finer approximations of the truth $\para_0$, as the algorithm proceeds (or more precisely, as $m$ grows). 

\section{Theoretical Results and Simulations} \label{analysis}
We start by establishing performance guarantees on the regret and the learning accuracy for bootstrap-based policies, supplemented by numerical examples that illustrate the behavior of Algorithm \ref{adaptivealgo1} for both identification and regulation. The following result specifies the growth rate of the regret $\regret{n}{\policy}$ for the policy $\policy$ designed by Algorithm \ref{adaptivealgo1}, as well as the decay rate of the identification error $\Mnorm{\estpara{n}-\para_0}{2}$.
\begin{thm} \label{asympTheorem}
	Letting $\policy$ be the policy given by Algorithm \ref{adaptivealgo1}, define the learned parameter $\estpara{n}$ by \eqref{LSE}. Then, we have
	\begin{eqnarray*}
		\limsup\limits_{n \to \infty} \left({{n^{-1/2} \log^{-2} n}}\right){\regret{n}{\policy}} &<& \infty , \label{upperb}\\
		\limsup\limits_{n \to \infty} \left({{n^{1/2} \log^{-2} n}}\right){\Mnorm{\estpara{n}-\para_0}{}^2} &<& \infty. \label{learningrate}
	\end{eqnarray*}
\end{thm} 
\ifarxiv The proof is provided in the appendix\fi. Technically, it relies on the careful examination of the effect of Algorithm~\ref{bootalgo} on the randomization of the feedback gains $\Lmatrix{\boot{\para}_{\lceil \rrate^i \rceil}}$. This randomization in turn diversifies the extended gain matrices $\left\{\extendedLmatrix{\boot{\para}_{\lceil \rrate^i \rceil}}\right\}_{i=1}^m$, so that their superposition efficiently explores the whole parameter space $\R^{p \times q}$, as $m$ grows. To this end, we utilize the state-of-the-art results on the behavior of the algebraic Riccati equation~\cite{faradonbeh2017finite}, properties of the optimality manifold~\cite{polderman1986necessity,polderman1986note,faradonbeh2018optimality}, and results from martingale theory~\cite{lai1983asymptotic,dani2008stochastic,abbasi2011improved}, limit distributions of dependent sequences~\cite{brown1971martingale,mcleish1974dependent}, and the bootstrap~\cite{hall2013bootstrap}.

The regulation and identification rates of Theorem \ref{asympTheorem} are modulo logarithmic factors similar to the corresponding rates of the reinforcement learning policies utilizing OFU \cite{abbasi2011regret,faradonbeh2017finite}, additive randomization~\cite{faradonbeh2018optimality}, posterior sampling~\cite{abeille2018improved}, and input perturbation~\cite{faradonbeh2018input}. Moreover, the square root scaling of the regret is efficient for adaptive regulation of LQ systems as discussed next. 
\begin{table*}
	\begin{eqnarray}
	A_0 &=& \begin{bmatrix}
	1.07     &    0 &  -0.37\\
	0.48  & -0.89   & 0.85\\
	0  &  0.04 &  -0.93
	\end{bmatrix}, \:\:\:\:\:\:
	B_0= \begin{bmatrix}
	-0.48  &  0.44  & -0.30\\
	-0.52  &  0.59  &  0.26\\
	0.30    &     0  & -0.74
	\end{bmatrix} \label{dynamicsmatrices}\\
	Q_x &=& \begin{bmatrix}
	0.65  & -0.08  & -0.14\\
	-0.08   & 0.57  &  0.26\\
	-0.14   & 0.26&    1.00
	\end{bmatrix},\:\:\:\:\:\:
	Q_u = \begin{bmatrix}
	0.20 &   0.05  &  0.09\\
	0.05  &  0.14 &   0.04\\
	0.08   & 0.04&    0.28
	\end{bmatrix} \label{costmatrices}\\
	\Kmatrix{\para_0} &=& \begin{bmatrix}
	0.94   & 0.06  & -0.32\\
	0.06  &  0.88 &   0.02\\
	-0.32   & 0.02  &  1.37
	\end{bmatrix},\:\:\:\:\:\:
	\Lmatrix{\para_0}= \begin{bmatrix}
	0.64  & -0.13 &   0.44\\
	-0.71  &  0.63 &  -0.11\\
	0.22   & 0.08 &  -0.91
	\end{bmatrix} \label{KLmatrices}
	\end{eqnarray}
\end{table*}

Recalling the discussion at the beginning of Section \ref{algos}, an adaptive control policy needs to sufficiently explore the parameter space in order to balance the trade-off of identification and regulation. For falsifying the imprecise approximation $\estpara{n}$ through an exploration procedure, the control signals $\left\{\action{t}\right\}_{t=0}^{n-1}$ need to deviate from the optimal feedback gain $\Lmatrix{\para_0}$. More precisely, for a policy $\policy$, let the deviations from the optimal feedback be $\epsilon_t= \norm{\action{t}-\Lmatrix{\para_0}\state{t}}{2}$, for $0 \leq t <n$. Then, observing the history $\history{n}$, the error of estimating the true dynamics parameter $\para_0$ is at least $\sigma_n$ (modulo a constant factor), where $\sigma_n^{-2}=\sum\limits_{t=0}^{n-1} \epsilon_t^2$~\cite{lai1986asymptotically,sarkar2018fast}. Hence, if $\policy$ aims to falsify $\estpara{n}$, the difference $\Mnorm{\estpara{n}-\para_0}{2}$ needs to be in the order of magnitude at least $\sigma_n$~\cite{bittanti2006adaptive,faradonbeh2018optimality}. Whenever $\policy$ employs $\boot{\para}_n$ for designing control inputs, $\liminf\limits_{n \to \infty} \sigma_n^{-1} \Mnorm{\boot{\para}_n-\para_0}{2}>0$ holds, since $\boot{\para}_n$ needs to be found through $\estpara{n}$.  

On the other hand, for the above deviations we have
\begin{equation} \label{tightregret}
\liminf\limits_{n \to \infty} \sigma_n^2 \regret{n}{\policy} > 0,
\end{equation}
according to the regret specification recently established~\cite{faradonbeh2018optimality}. Further, applying the adaptive feedback $u(n)=\Lmatrix{\boot{\para}_n}x(n)$ at time $n$, the increase $\regret{n+1}{\policy}-\regret{n}{\policy}$ in the regret is approximately $\Mnorm{\Lmatrix{\boot{\para}_n}-\Lmatrix{\para_0}}{2}^2$~\cite{faradonbeh2018input}, which is up to a constant factor at least $\Mnorm{{\boot{\para}_n}-{\para_0}}{2}^2$~\cite{faradonbeh2017finite,faradonbeh2018optimality}. Thus, the lower bound $\sigma_n$ for $\Mnorm{\boot{\para}_n-\para_0}{2}$ implies that $\liminf\limits_{n \to \infty} \sigma_n^{-2} \left(\regret{n+1}{\policy}-\regret{n}{\policy}\right) > 0$. Putting the latter result and \eqref{tightregret} together, we obtain $\liminf\limits_{n \to \infty} \left(\regret{n+1}{\policy}^2 - \regret{n}{\policy}^2\right) > 0$, which provides the lower bound $\liminf\limits_{n \to \infty} n^{-1/2}\regret{n}{\policy}> 0$. Note that a rigorous proof of the above lower bound argument is beyond the scope of this paper. For more detailed discussions, we refer the reader to the aforementioned references.
\begin{figure} 
	\centering
	\scalebox{.7}
	{\includegraphics {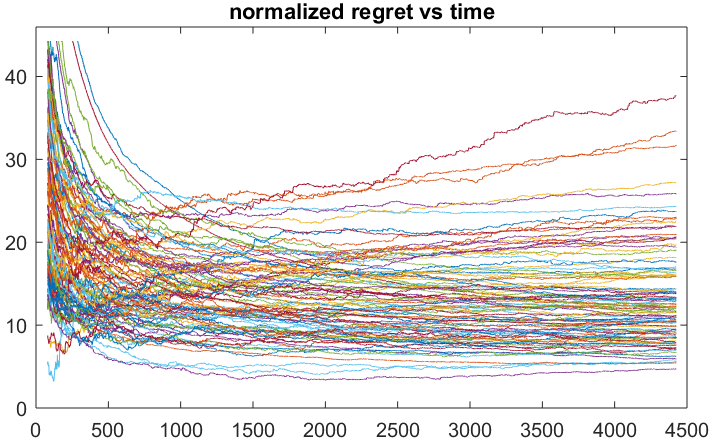}} 
	\caption{Normalized regret $n^{-1/2} \regret{n}{\policy}$ vs $n$, for Algorithm~\ref{adaptivealgo1} with $\rrate=1.2$.}
	\label{RBregret}
\end{figure}
\begin{figure} 
	\centering
	\scalebox{.7}
	{\includegraphics {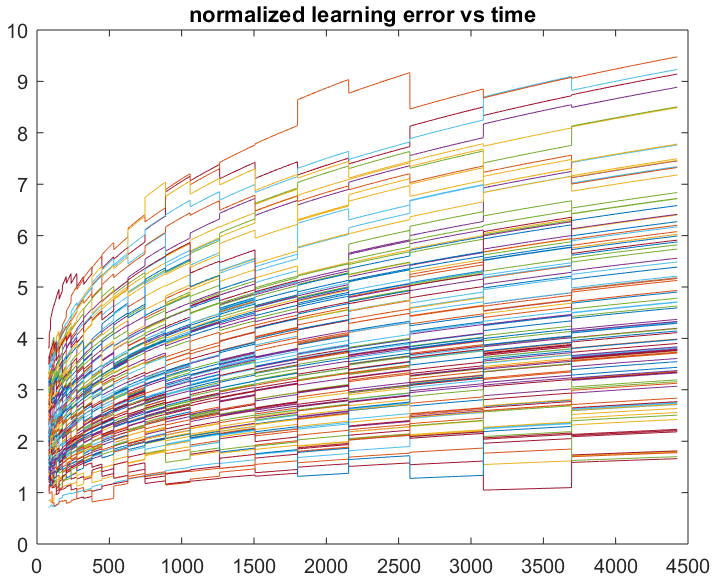}} 
	\caption{Normalized identification error $n^{1/4} \Mnorm{\estpara{n}-\para_0}{2}$ vs $n$, for Algorithm~\ref{adaptivealgo1} with $\rrate=1.2$.}
	\label{RB_est_error}
\end{figure}
\subsection{Numerical Illustration}
Next, we present numerical analyses employing Algorithm \ref{adaptivealgo1} for decision-making under uncertainty. Henceforth, let $\policy$ be the reinforcement learning policy provided by Algorithm \ref{adaptivealgo1}, with reinforcement rate $\rrate=1.2$. The true dynamical model and cost matrices are provided in \eqref{dynamicsmatrices} and \eqref{costmatrices}, respectively. Solving the algebraic Riccati equation, we get $\Kmatrix{\para_0}, \Lmatrix{\para_0}$ given in \eqref{KLmatrices}, which lead to a closed-loop matrix of the spectral radius $\eigmax{\para_0 \extendedLmatrix{\para_0}}=0.26$. 
\begin{figure*} 
	\centering
	\scalebox{.45}
	{\includegraphics {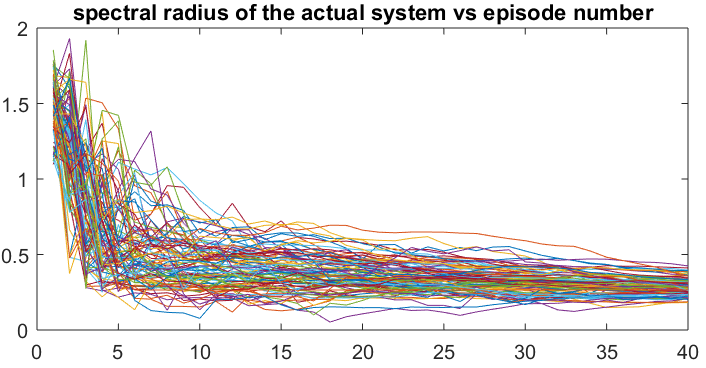}} 
	\scalebox{.45}
	{\includegraphics {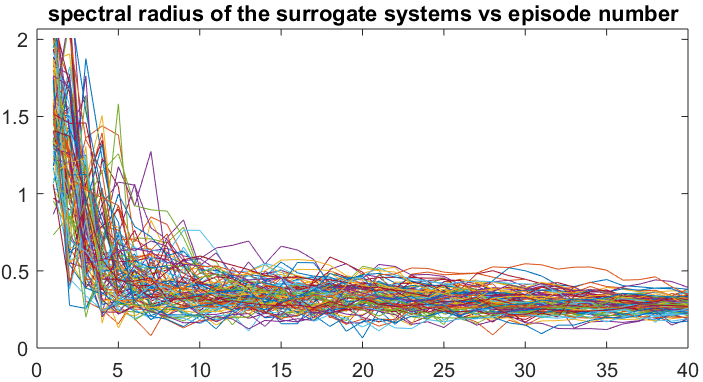}} 
	\caption{Stability of the closed-loop matrices for Algorithm~\ref{adaptivealgo1} with $\rrate=1.2$: the spectral radius of the actual system $\eigmax{\para_0 \extendedLmatrix{\boot{\para}_{\lceil \rrate^m \rceil}}}$, and the surrogate system $\eigmax{\estpara{\lceil \rrate^m \rceil} \extendedLmatrix{\boot{\para}_{\lceil \rrate^m \rceil}}}$, are reported as functions of $m$.}
	\label{eigmax}
\end{figure*}

Figure \ref{RBregret} depicts the normalized regret as a function of $n$, for $100$ replicates of the stochastic linear system in \eqref{systemeq1}. The corresponding normalized identification errors are plotted in Figure \ref{RB_est_error}. These figures are in a full agreement with the theoretical result of Theorem~\ref{asympTheorem}; both normalized rates $n^{-1/2} \regret{n}{\policy}$ and $n^{1/4} \Mnorm{\estpara{n}-\para_0}{2}$ are dominated by logarithmic factors of the time index $n$. In Figure \ref{eigmax}, we plot the resulting spectral radius of the reinforcement learning policy $\policy$ for both the actual system of the dynamics parameter $\para_0$, as well as that of the surrogate system of $\estpara{n}$. According to Figure \ref{eigmax}, Algorithm~\ref{adaptivealgo1} fully stabilizes the system, even though in the first few episodes the system is unstable. 
\begin{figure} 
	\centering
	\scalebox{.47}
	{\includegraphics {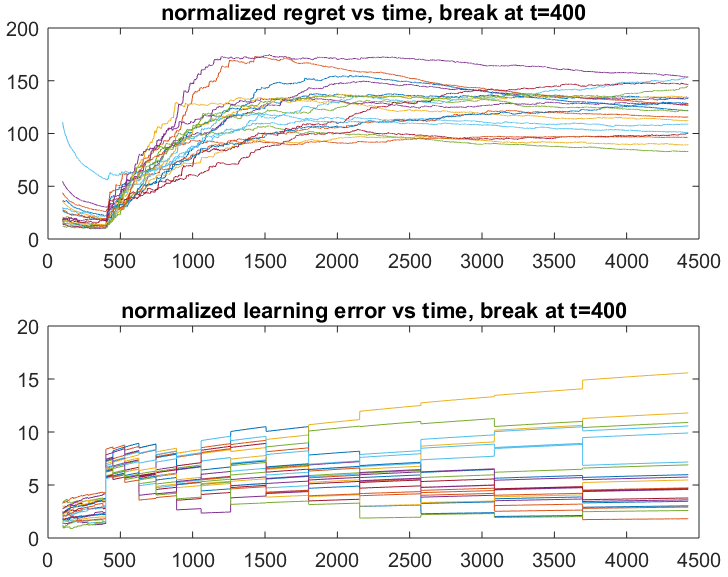}} 
	\caption{Normalized regret $n^{-1/2} \regret{n}{\policy}$ and normalized learning error $n^{1/4} \Mnorm{\estpara{n}-\para_0}{2}$ vs $n$. Algorithm~\ref{adaptivealgo1} is run with $\rrate=1.2$, while a break occurs at time $t=400$.}
	\label{1break}
\end{figure}
The ensuing figures indicate the robustness of Algorithm \ref{adaptivealgo1} to structural breaks. Figure \ref{1break} shows the performance of $n^{-1/2} \regret{n}{\policy}$ and $n^{1/4} \Mnorm{\estpara{n}-\para_0}{2}$ for a single break in the model, wherein at time $t =400$ the dynamics matrices suddenly become
\begin{equation*}
{\small A_0 = \begin{bmatrix}
	1.07     &    0 &  -0.37\\
	0.48  & -0.89   & 0.85\\
	0.44  &  0.04 &  0
	\end{bmatrix}, B_0= \begin{bmatrix}
	-0.48  &  0.44  & -0.30\\
	-0.52  &  0.59  &  0.26\\
	0.30    &    -0.44  & 0
	\end{bmatrix}}.
\end{equation*}

A similar performance analysis while the system incurs two breaks is provided in Figure~\ref{2breaks}. The first break is similar to the one mentioned above, and occurs at time $t=200$. Then, for the second break at time $t=700$, the true dynamics matrices change to
\begin{equation*}
{\small A_0 = \begin{bmatrix}
	1.07     &    0 &  -1.04\\
	0.48  & -0.89   & 0.85\\
	0.44  &  0.81 &  0
	\end{bmatrix}, B_0= \begin{bmatrix}
	-0.48  &  0.44  & -0.30\\
	-0.52  &  0.59  &  -0.26\\
	0.30    &    -0.30  & 0
	\end{bmatrix}}.
\end{equation*}
\begin{figure} 
	\centering
	\scalebox{.47}
	{\includegraphics {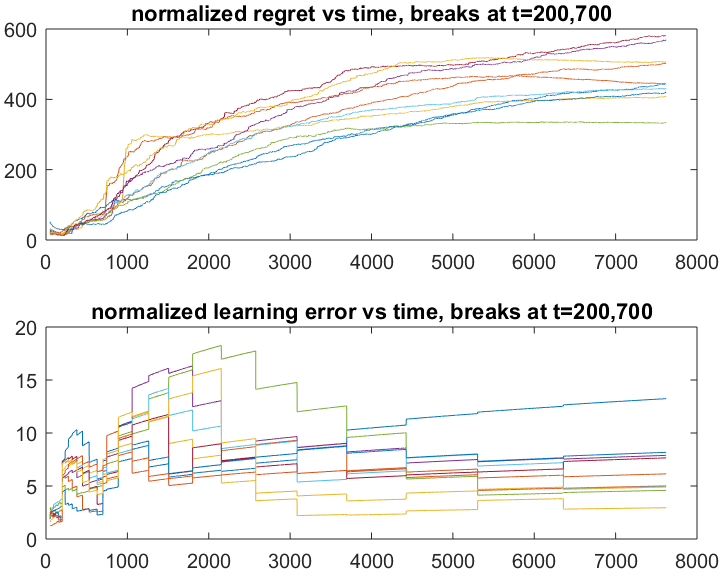}} 
	\caption{Nornmalized regret $n^{-1/2} \regret{n}{\policy}$ and normalized learning error $n^{1/4} \Mnorm{\estpara{n}-\para_0}{2}$ vs $n$. Algorithm~\ref{adaptivealgo1} is run with $\rrate=1.2$, while two breaks occur at times $t=200$, $t=700$.}
	\label{2breaks}
\end{figure}

According to Figures~\ref{1break} and~\ref{2breaks}, Algorithm~\ref{adaptivealgo1} is robust to remarkably large values of model mis-specifications. Note that the reinforcement learning policy is fully ignorant of the breaks. So, $\policy$ needs to adaptively adjust its decision-making law toward the new optimal policies.

The rationale for the exhibited robustness is as follows: when a break occurs, the parameter estimate $\estpara{}$ becomes an inaccurate approximation of the true dynamics matrix $\para_0$. This in turn leads the regression residuals $\left\{\residunoise{t}\right\}_{t=1}^n$ becoming large. Therefore, the bootstrapped parameter $\boot{\para}$ computed by Algorithm \ref{bootalgo} provides a large randomization, which in turn leads to an increase in the exploration phase. Then, after a few episodes, the resulting enhanced exploration provides more accurate estimates $\estpara{}$, and the above {\em negative feedback} procedure proceeds. Thus, as time grows, Algorithm~\ref{adaptivealgo1} \emph{self-tunes} to the equilibrium of the suitable amount of exploration. 

The above argument intuitively indicates that the aforementioned equilibrium is a stable one. Since the endogenous randomization of the bootstrap procedure consistently assesses the accuracy of the fitted model $\estpara{}$, the resulting adaptive policy automatically adjusts the old decision-making strategy to the new environment. Hence, the algorithm accordingly addresses the unexpected flaw of the \emph{sudden and unknown} changes in the true model $\model{M}\left( \para_0 \right)$, as well as the resulting unpredicted deviations in the trajectory of the state sequence $\left\{ \state{t} \right\}_{t=0}^\infty$. 

\section{Concluding Remarks}
We proposed a reinforcement learning algorithm for sequential decision-making for an LQ system with unknown temporal dynamics. The presented model-based policy is based on residual bootstrap, and is shown to be efficient in terms of both identification and regulation. Namely, we establish the rates for the worst-case regret, as well as the learning accuracy. Further, we discussed the robustness of the bootstrap method for handling unexpected changes in the dynamical model. 

As the first work on bootstrap-based policies for LQ models, it poses a number of interesting questions. For example, theoretical analysis for addressing the performance of bootstrap method under \emph{imperfect} observation is a natural direction for future work. Further subjects of interest include design and analysis of fully non-parametric randomization methods such as {\em covariate resampling}. Finally, extending the presented framework to {\em model-free} algorithms can be considered as another fruitful research direction to examine.

\ifarxiv
\appendices
\section{Auxiliary Results} \label{AuxResults}
Next, we present auxiliary results being used in the proof of Theorem \ref{asympTheorem} in Appendix \ref{ProofApp}. First, we present the regret bounds of Lemma \ref{generalregret}, for which the proof can be found in the work of Faradonbeh et al.~\cite{faradonbeh2018optimality}, using martingale convergence analysis of Lai and Wei~\cite{lai1982least}. Then, Lemma \ref{optmanifold} provides the local specification of the optimality manifold which is established for both full-rank \cite{polderman1986necessity,polderman1986note} and rank-deficient dynamics matrices \cite{faradonbeh2018optimality}. Subsequently, Lemma \ref{Lailemma} is presented to study convergence rates of linear regression procedures. The asymptotic~\cite{lai1983asymptotic} and non-asymptotic~\cite{dani2008stochastic,abbasi2011improved} proofs of Lemma~\ref{Lailemma} are available in the literature. 

Then, we state Lemma \ref{smallcov} which addresses the behavior of the empirical covariance matrix of the state sequence of a stabilized system~\cite{lai1985asymptotic}. Further, we have Lemma \ref{Lipschitzlemma} that is providing the local Lipschitz continuity property (i.e. in a neighborhood of $\para_0$) of the feedback matrix $\Lmatrix{\para}$~\cite{faradonbeh2017finite}. Finally, we establish Lemma \ref{bootcov} on the population covariance matrix induced by the empirical probability measure $\Pboot{n}$ defined in \eqref{empiricaldist}.

\begin{lem} \label{generalregret}
	For the sequence of $r \times p$ matrices $\left\{ G_t \right\}_{t=0}^\infty$, suppose that there is a filtration $\filter{t}$ such that for all $t \geq 0$, $\state{t},G_t,\noise{t}$ are $\mathcal{F}_t$-measurable, and $\E{\noise{t+1} \big| \mathcal{F}_t} = 0$. Then, for the regret of the policy 
	\begin{eqnarray*}
	\policy: \:\:\: \action{t}= G_t \state{t}, \:\:\: t=0,1,2,\cdots,
	\end{eqnarray*}
	the following holds:
	\begin{eqnarray*}
		\limsup\limits_{n \to \infty} \frac{\regret{n}{\policy}}{\sum\limits_{t=0}^{n-1} \norm{ \left(\Lmatrix{\para_0} - G_t \right) \state{t} }{}^2+ \norm{\sum\limits_{t=1}^n \left(A_0+B_0\Lmatrix{\para_0}\right)^{n-t}\noise{t}}{2}^2} <\infty .
	\end{eqnarray*}
\end{lem}

\begin{lem} \label{optmanifold}
	For the stabilizable parameter $\para_1=\left[ A_1,B_1 \right]$, let $\levelset{\para_1}$ be the manifold of optimal feedback gains:
	\begin{eqnarray*} \label{levelsetdeff}
	\levelset{\para_1} = \left\{ \para \in \R^{p \times q} : \Lmatrix{\para} = \Lmatrix{\para_1} \right\}.
	\end{eqnarray*}
	Then, the tangent space of $\levelset{\para_1}$ at point $\para_1$ consists of matrices $\left[ M,N \right]$, such that $M \in \R^{p \times p}, N \in \R^{p \times r}$ satisfy
	\begin{eqnarray*} \label{tangentspace}
	N' \Kmatrix{\para_1}D_1 + B_1' Z + B_1' \sum\limits_{k=0}^\infty {D_1'}^k \left(D_1'Z+Z'D_1\right) {D_1}^{k+1} =0_{r \times p},
	\end{eqnarray*}
	where $D_1 = \para_1 \extendedLmatrix{\para_1}$, $Z = \Kmatrix{\para_1} \left( M + N\Lmatrix{\para_1} \right)$.
\end{lem}

\begin{lem} \label{Lailemma}
	Consider the dynamical system $\state{t+1}=\para_0 \extendedL{t}\state{t}+\noise{t+1}$. Then, define
	\begin{eqnarray*}
		U_n &=& \sum\limits_{t=0}^{n-1} \extendedL{t} \state{t} \state{t}' \extendedL{t}'  \in \R^{ q \times q}, \\
		W_n &=& \sum\limits_{t=0}^{n-1} \noise{t+1} \state{t}' \extendedL{t}' U_n^{-1/2} \in \R^{p \times q},
	\end{eqnarray*}
	where $U_n^{-1}$ is the Moore-Penrose inverse. It holds that
	\begin{eqnarray*}
		\limsup\limits_{n \to \infty} \frac{\eigmax{W_n W_n'}}{\log \eigmax{U_n}} < \infty.
	\end{eqnarray*}
	Further, since $\estpara{n}$ provided by \eqref{LSE} satisfies the normal equation $\estpara{n}U_n = \sum\limits_{t=0}^{n-1} \state{t+1} \state{t}' \extendedL{t}'$, we have 
	\begin{equation*}
	\left( \estpara{n} - \para_0 \right) U_n \left( \estpara{n} - \para_0 \right)' = W_n W_n'.
	\end{equation*}
	Therefore, we get 
	\begin{eqnarray*}
	\limsup\limits_{n \to \infty} \frac{\eigmin{U_n}\Mnorm{\estpara{n} - \para_0}{}^2}{\log \eigmax{U_n}} < \infty. 
	\end{eqnarray*}
	
\end{lem}

\begin{lem} \label{smallcov}
	Suppose that the control feedback matrix $\Lmatrix{\boot{\para}}$ is applied to the dynamical model $\model{M}\left( \para_0 \right)$. Hence, plugging $\action{t}=\Lmatrix{\boot{\para}}\state{t}$ in \eqref{systemeq1}, the system evolves according to $\state{t+1}=\para_0 \extendedLmatrix{\boot{\para}}\state{t}+\noise{t+1}$. Assuming $\eigmax{\para_0 \extendedLmatrix{\boot{\para}}}<1$, for the empirical covariance matrix $V_n=\sum\limits_{t=0}^n \state{t}\state{t}'$ the following holds:
	\begin{eqnarray*}
		\lim\limits_{n \to \infty} n^{-1}V_n = \sum\limits_{k=0}^\infty \left( \para_0 \extendedLmatrix{\boot{\para}} \right)^k \covmat{} {\left( \para_0 \extendedLmatrix{\boot{\para}} \right)'}^k.
	\end{eqnarray*}
\end{lem}

\begin{lem} \label{Lipschitzlemma}
	Letting $\ssconstant_P<\infty$, there is a constant $\ssconstant_G<\infty$ such that 
	\begin{equation*}
		\sup\limits_{ \para : \Mnorm{\Kmatrix{\para}}{2} \leq \ssconstant_P } \frac{\Mnorm{\Lmatrix{\para}-\Lmatrix{\para_0}}{}}{\Mnorm{\para-\para_0}{}} < \ssconstant_G.
	\end{equation*}
\end{lem}

\begin{lem} \label{bootcov}
	Let $\Eboot{n}$ be the expectation with respect to $\Pboot{n}$; the empirical probability measure of the residuals defined in \eqref{empiricaldist}. Then, for the Bootstrap covariance matrix $\empiricalcovmat{n} = \Eboot{n} \left[\bootnoise{t} \bootnoise{t}'\right]$ we have
	\begin{equation} \label{PDBootCov}
	0 < \liminf\limits_{n \to \infty} \eigmin{\empiricalcovmat{n}} \leq \limsup\limits_{n \to \infty} \eigmax{\empiricalcovmat{n}}<\infty.
	\end{equation}
\end{lem}

\begin{proof}
	First, the definition of $\left\{ \residunoise{t} \right\}_{t=1}^n$ in \eqref{residuals}, in addition to the dynamics \eqref{systemeq1} yield $\residunoise{t+1}= \left(\para_0-\estpara{n}\right) \extendedL{t} \state{t} + \noise{t+1}$. Defining $U_n,W_n$ similar to Lemma~\ref{Lailemma}, the normal equation $\estpara{n}U_n = \sum\limits_{t=0}^{n-1} \state{t+1} \state{t}' \extendedL{t}'$ implies that 
	\begin{eqnarray*}
		\sum\limits_{t=0}^{n-1} \left[\noise{t+1} \state{t}' \extendedL{t}'\left(\para_0-\estpara{n}\right)' + \left(\para_0-\estpara{n}\right) \extendedL{t} \state{t} \noise{t+1}'\right]=-2W_nW_n'.
	\end{eqnarray*}
	So, we obtain
	\begin{eqnarray*}
		 \empiricalcovmat{n} &=& \frac{1}{n} \sum\limits_{t=0}^{n-1} \residunoise{t+1} \residunoise{t+1}' - \avenoise{n} \: \avenoise{n}' = \frac{1}{n} \sum\limits_{t=1}^{n} \noise{t} \noise{t}' - \frac{1}{n} W_n W_n' - \avenoise{n} \: \avenoise{n}',
	\end{eqnarray*}
	since $\left(\para_0-\estpara{n}\right) U_n \left(\para_0-\estpara{n}\right)'= W_n W_n'$. Thus, applying the Law of Large Numbers to the matrices $\noise{t}\noise{t}'$, we get $\limsup\limits_{n \to \infty} \eigmax{\empiricalcovmat{n}} \leq \eigmax{\covmat{}}$. Further, by the Law of Large Numbers, $n^{-1} \sum\limits_{t=1}^{n} \noise{t}$ vanishes as $n$ grows. Therefore, the definitions of $U_n,W_n$ lead to
	\begin{eqnarray*}
		\limsup\limits_{n \to \infty} \norm{\avenoise{n}}{2} &\leq& \limsup\limits_{n \to \infty} \frac{1}{n} \sum\limits_{t=0}^{n-1} \norm{ \left(\para_0-\estpara{n}\right) \extendedL{t} \state{t}}{2} + \norm{\frac{1}{n} \sum\limits_{t=0}^{n-1} \noise{t+1}}{2} \\ 
		&\leq& \limsup\limits_{n \to \infty} \left(\frac{1}{n} \sum\limits_{t=0}^{n-1} \norm{\left(\para_0-\estpara{n}\right) \extendedL{t} \state{t}}{2}^2\right)^{1/2} \\
		&\leq& \limsup\limits_{n \to \infty} \tr{\frac{1}{n} \sum\limits_{t=0}^{n-1} \left(\para_0-\estpara{n}\right) \extendedL{t} \state{t} \state{t}' \extendedL{t}' \left(\para_0-\estpara{n}\right)'}^{1/2} \\
		&\leq& \limsup\limits_{n \to \infty} \left(\frac{p}{n} \eigmax{W_n W_n'}\right)^{1/2}.
	\end{eqnarray*}
	Finally, since Lemma \ref{Lailemma} implies that $\limsup\limits_{n \to \infty} n^{-1} \eigmax{W_n W_n'}=0$, we get the desired result on the smallest eigenvalue: $\liminf\limits_{n \to \infty} \eigmin{\empiricalcovmat{n}} \geq \eigmin{\covmat{}}$. 
\end{proof}

\section{Proof of Theorem \ref{asympTheorem}} \label{ProofApp}
	The following analysis rigorously studies the behavior of both $\regret{n}{\policy}$ and $\Mnorm{\estpara{n}-\para_0}{2}$ as the time of interacting with the system, $n$, grows. In the sequel, we assume that the system is stable. The stabilization problem has been addressed previously in the literature ~\cite{faradonbeh2018stabilization}. In fact, an ad-hoc algorithm for stabilizing the system is presented and analyzed in the work of Faradonbeh et al.~\cite{faradonbeh2018stabilization}. It establishes high probability guarantees for stabilization in finite time, using random feedback gains. The proposed method can be executed a priori, and terminates in a relatively short time period. Moreover, the random feedback framework presented for stabilization algorithms can also be implemented with the bootstrap method of Algorithm \ref{bootalgo}. Therefore, the subsequent theoretical analysis focuses on establishing regret bounds and learning accuracy of Algorithm~\ref{adaptivealgo1} after the transient stabilization period.  
	
	In the reinforcement learning policy provided by Algorithm \ref{adaptivealgo1}, the residual bootstrap procedure of Algorithm \ref{bootalgo} is being called at the end of every episode. Fixing $i$, for $j<i$ define the following quantities when the algorithm $\text{BOOTSTRAP}$ is called at time $t = \lceil \rrate^i \rceil$:
	\begin{eqnarray*}
	\boot{W}_{i} &=& \sum\limits_{t=0}^{\lceil \rrate^{i} \rceil-1} \bootnoise{t+1} \bootstate{t}' \extendedLmatrix{\boot{\para}_t}' \boot{U}_i^{-1/2} \in \R^{p \times q}, \\
	\boot{V}_{ij} &=& \sum\limits_{t=\lceil \rrate^{j} \rceil}^{\lceil \rrate^{j+1} \rceil -1} \bootstate{t} \bootstate{t}' \in \R^{p \times p}, \\
	\randommatrix_{ij} &=& \sum\limits_{t=\lceil \rrate^{j} \rceil }^{\lceil \rrate^{j+1} \rceil-1} \bootnoise{t+1} \bootstate{t}' \in \R^{p \times p} ,\\
	\boot{U}_{i} &=& \sum\limits_{t=0}^{\lceil \rrate^{i} \rceil-1} \extendedLmatrix{\boot{\para}_t} \bootstate{t} \bootstate{t}' \extendedLmatrix{\boot{\para}_t}' \\
	&=& \sum\limits_{j=0}^{i-1} \extendedLmatrix{\boot{\para}_{\lceil \rrate^{j} \rceil}} \boot{V}_{ij} \extendedLmatrix{\boot{\para}_{\lceil \rrate^{j} \rceil}}'  \in \R^{ q \times q}.
	\end{eqnarray*}
	Note that the notation is slightly overloaded in the above expressions since in every call, $\text{BOOTSTRAP}$ generates a completely new set of surrogate noise and state vectors $\bootnoise{\cdot},\bootstate{\cdot}$. Similarly, for the original system of dynamics parameter $\para_0$ define the matrices
	\begin{eqnarray*}
	W_i &=& \sum\limits_{t=0}^{\lceil \rrate^i \rceil -1} \noise{t+1} \state{t}' \extendedLmatrix{\boot{\para}_t}' U_i^{-1/2} \in \R^{p \times q}, \\
	{V}_{j} &=& \sum\limits_{t=\lceil \rrate^{j} \rceil}^{\lceil \rrate^{j+1} \rceil -1} \state{t} \state{t}' \in \R^{p \times p},\\
	U_i &=& \sum\limits_{j=0}^{i-1} \extendedLmatrix{\boot{\para}_{\lceil \rrate^{j} \rceil}} {V}_{j} \extendedLmatrix{\boot{\para}_{\lceil \rrate^{j} \rceil}}'  \in \R^{ q \times q}. \label{grammatrices}
	\end{eqnarray*}
	
Using the martingale Central Limit Theorem, $\rrate^{-i/2} \randommatrix_i$ converges in distribution to a Gaussian random matrix. Moreover, for an arbitrary $j < i$, let $\mathcal{F}_j$ be the sigma-field generated by $\left\{ \noise{t} \right\}_{t=1}^{\lceil \rrate^j \rceil}$ and $\left\{ \left\{ \bootnoise{t} \right\}_{t=1}^{\lceil \rrate^\ell \rceil} \right\}_{\ell=1}^j$. Then, for the $\mathcal{F}_j$-measurable $v \in \R^{p}$, Lemma~\ref{smallcov} and Lemma~\ref{bootcov} imply that as $k$ grows,
\begin{eqnarray} \label{MCLT}
	\rrate^{-k/2} \randommatrix_{ik}v \Rightarrow \mathcal{N} \left( 0, \Sigma_v \right),
\end{eqnarray}
where ``$\Rightarrow$" denotes the convergence in distribution, and $\inf\limits_{\norm{v}{2}=1} \eigmin{\Sigma_v}>0$. 
	Next, we study eigenvalues of the matrix $\boot{U}_i$. According to Lemma~\ref{smallcov} and Lemma~\ref{bootcov}, we have:
	\begin{eqnarray} \label{sufficientmineig}
	\liminf\limits_{i \to \infty} \frac{\eigmin{\boot{U}_i}}{\inf\limits_{ \norm{v}{2}=1}  \sum\limits_{j=0}^{i-1} \rrate^{j} \norm{\extendedLmatrix{\boot{\para}_{\lceil \rrate^{j} \rceil}}'v}{2}^2}  > 0, \:\:\:\:\:\: \limsup\limits_{i \to \infty} \frac{\eigmax{\boot{U}_i}}{\sup\limits_{ \norm{v}{2}=1}  \sum\limits_{j=0}^{i-1} \rrate^{j} \norm{\extendedLmatrix{\boot{\para}_{\lceil \rrate^{j} \rceil}}'v}{2}^2} < \infty.
	\end{eqnarray}
	In order to establish a lower bound for the smallest eigenvalue of $\boot{U}_i$, compute ${\boot{U}_{i}}^{-1}$ according to the blocks of 
	\begin{eqnarray*}
		\boot{U}_{i} = \sum\limits_{j=0}^{i-1} \begin{bmatrix}
			 \boot{V}_{ij} & \boot{V}_{ij} \Lmatrix{\boot{\para}_{\lceil \rrate^{j} \rceil}}' \\
			\Lmatrix{\boot{\para}_{\lceil \rrate^{j} \rceil}} \boot{V}_{ij} & \Lmatrix{\boot{\para}_{\lceil \rrate^{j} \rceil}} \boot{V}_{ij} \Lmatrix{\boot{\para}_{\lceil \rrate^{j} \rceil}}'
		\end{bmatrix}.
	\end{eqnarray*}
	So, 
	\begin{eqnarray*}
		{\boot{U}_{i}}^{-1}= \begin{bmatrix}
			 \left( \sum\limits_{j=0}^{i-1} \boot{V}_{ij} \right)^{-1} + X_{12}(i) {X_{22}(i)}^{-1} {X_{12}(i)}'& -X_{12}(i){X_{22}(i)}^{-1} \\
			-{X_{22}(i)}^{-1} {X_{12}(i)}' & {X_{22}(i)}^{-1}
		\end{bmatrix},
	\end{eqnarray*}
	where 
	\begin{eqnarray}
		X_{12}(i) &=& \left( \sum\limits_{j=0}^{i-1} \boot{V}_{ij} \right)^{-1} \sum\limits_{j=0}^{i-1} \boot{V}_{ij} \Lmatrix{\boot{\para}_{\lceil \rrate^j \rceil}}', \notag \\
		X_{22}(i) &=& \sum\limits_{j=0}^{i-1} \Lmatrix{\boot{\para}_{\lceil \rrate^j \rceil}} \boot{V}_{ij} \Lmatrix{\boot{\para}_{\lceil \rrate^j \rceil}}' - \left( \sum\limits_{j=0}^{i-1} \Lmatrix{\boot{\para}_{\lceil \rrate^j \rceil}} \boot{V}_{ij}\right) \left( \sum\limits_{j=0}^{i-1} \boot{V}_{ij} \right)^{-1} \sum\limits_{j=0}^{i-1} \boot{V}_{ij} \Lmatrix{\boot{\para}_{\lceil \rrate^j \rceil}}' \notag \\
		&=& \sum\limits_{j=0}^{i-1} \left[\Lmatrix{\boot{\para}_{\lceil \rrate^j \rceil}}' - X_{12}(i)\right]' \boot{V}_{ij} \left[\Lmatrix{\boot{\para}_{\lceil \rrate^j \rceil}}' - X_{12}(i)\right] . \label{X22vsX12}
	\end{eqnarray}
Clearly, $\eigmin{X_{22}(i)} \geq \eigmin{\boot{U}_i}$. Further, let $v = \left[v_1',v_2'\right]' \in \R^q$ be arbitrary, where $v_1 \in \R^p$, $v_2 \in \R^r$. Then, we have
\begin{eqnarray*}
		v' \boot{U}_i v &=& \sum\limits_{j=0}^{i-1} 
		\left[ v_1 + \Lmatrix{\boot{\para}_{\lceil \rrate^{j} \rceil}}'v_2 \right]' \boot{V}_{ij} \left[ v_1 + \Lmatrix{\boot{\para}_{\lceil \rrate^{j} \rceil}}'v_2 \right]  \notag \\
		&=& v_2' X_{22}(i) v_2 + \sum\limits_{j=0}^{i-1} 
		\left[ v_1 + X_{12}(i) v_2 \right]' \boot{V}_{ij} \left[ v_1 + X_{12}(i)v_2 \right].
\end{eqnarray*}
So, Lemma~\ref{smallcov} and Lemma~\ref{bootcov} lead to
\begin{eqnarray} \label{blockeigmineq}
	\liminf\limits_{i \to \infty} \inf\limits_{ \norm{v}{}=1} \frac{v' \boot{U}_i v}{v_2' X_{22}(i) v_2 + \rrate^i \norm{v_1 + X_{12}(i)v_2}{2}^2 } > 0.
\end{eqnarray}

On the other hand, the least squares estimates in \eqref{LSE}, \eqref{BootLSE} lead to 
\begin{eqnarray*}
	\left( \estpara{\lceil \rrate^{i} \rceil} - \para_0 \right) U_{i} \left( \estpara{\lceil \rrate^{i} \rceil} - \para_0 \right)' &=& W_{i} W_{i}' , \\
	\left( \boot{\para}_{\lceil \rrate^{i} \rceil} - \estpara{\lceil \rrate^{i} \rceil} \right) \boot{U}_{i} \left( \boot{\para}_{\lceil \rrate^{i} \rceil} - \estpara{\lceil \rrate^{i} \rceil} \right)' &=& \boot{W}_{i} \boot{W}_{i}'.
\end{eqnarray*}

Therefore, applying Lemma~\ref{Lailemma}, Lemma~\ref{smallcov}, and Lemma~\ref{bootcov}, we obtain
\begin{eqnarray*}
	\limsup\limits_{j < i; \:\: j \to \infty} i^{-1/2} \rrate^{j/2} \Mnorm{\left( \estpara{\lceil \rrate^{i} \rceil} - \para_0 \right) \extendedLmatrix{\boot{\para}_{\lceil \rrate^{j} \rceil}}}{2}  &<& \infty, \\
	\limsup\limits_{j < i; \:\: j \to \infty} i^{-1/2} \rrate^{j/2} \Mnorm{ \left( \boot{\para}_{\lceil \rrate^{i} \rceil} - \estpara{\lceil \rrate^{i} \rceil} \right) \extendedLmatrix{\boot{\para}_{\lceil \rrate^{j} \rceil}}}{2} &<& \infty.
\end{eqnarray*}
Combine the above two inequalities to get
\begin{eqnarray} \label{closed-loop-id}
	\limsup\limits_{j < i; \:\: j \to \infty} i^{-1/2} \rrate^{j/2} \Mnorm{ \left( \boot{\para}_{\lceil \rrate^{i} \rceil} - \para_0 \right) \extendedLmatrix{\boot{\para}_{\lceil \rrate^{j} \rceil}}}{2} &<& \infty.
\end{eqnarray}

Now, we compare $\Lmatrix{\boot{\para}_{\lceil \rrate^{i} \rceil}}$, $\Lmatrix{\boot{\para}_{\lceil \rrate^{j} \rceil}}$ based on the local specification of the optimality manifolds in Lemma~\ref{optmanifold}. To do so, we also need to characterize the right-hand-side $p \times r$ sub-matrix of the randomization matrix $\boot{\para}_{\lceil \rrate^{i} \rceil} - \estpara{\lceil \rrate^{i} \rceil}$ provided by the residual bootstrap procedure. Therefore, multiplying both sides of the normal equation $\boot{\para}_{\lceil \rrate^{i} \rceil} - \estpara{\lceil \rrate^{i} \rceil} = \boot{W}_i \boot{U}_{i}^{-1/2}$ by $\left[ 0_{r \times p}, I_r \right]'$, and using $\extendedLmatrix{\boot{\para}_{\lceil \rrate^{j} \rceil}}' \boot{U}_i^{-1} \left[ 0_{r \times p}, I_r \right]' =\left( -X_{12}(i) + \Lmatrix{\boot{\para}_{\lceil \rrate^{j} \rceil}}'\right) {X_{22}(i)}^{-1}$, we obtain
\begin{eqnarray*}
	\left(\boot{\para}_{\lceil \rrate^{i} \rceil} - \estpara{\lceil \rrate^{i} \rceil}\right) \left[ 0_{r \times p}, I_r \right]' &=& \boot{W}_i \boot{U}_{i}^{-1/2} \left[ 0_{r \times p}, I_r \right]' = \sum\limits_{t=0}^{\lceil \rrate^{i} \rceil-1} \bootnoise{t+1} \bootstate{t}' \extendedLmatrix{\boot{\para}_t}' \boot{U}_i^{-1} \left[ 0_{r \times p}, I_r \right]' \\
	&=& \sum\limits_{j=0}^{i-1} \randommatrix_{ij} \extendedLmatrix{\boot{\para}_{\lceil \rrate^{j} \rceil}}' \left(\sum\limits_{j=0}^{i-1} \extendedLmatrix{\boot{\para}_{\lceil \rrate^{j} \rceil}} \boot{V}_{ij} \extendedLmatrix{\boot{\para}_{\lceil \rrate^{j} \rceil}}'\right)^{-1} \left[ 0_{r \times p}, I_r \right]' \\
	&=& \sum\limits_{j=0}^{i-1} \randommatrix_{ij} \left( -X_{12}(i) + \Lmatrix{\boot{\para}_{\lceil \rrate^{j} \rceil}}'\right) {X_{22}(i)}^{-1}\\
	&=& \bigrandommat_i {X_{22}(i)}^{-1/2},
\end{eqnarray*} 
where $\bigrandommat_i=\sum\limits_{j=0}^{i-1} \randommatrix_{ij} \left( \Lmatrix{\boot{\para}_{\lceil \rrate^{j} \rceil}}'-X_{12}(i) \right) {X_{22}(i)}^{-1/2}$. Letting $v \in \R^q$ be an eigenvector of $\boot{U}_i$, decompose it to $v=\left[v_1',v_2'\right]'$, where $v_1 \in \R^p$, $v_2 \in \R^r$. Then, in order to consider $v' \boot{U}_{i+1} v$, \eqref{sufficientmineig} indicates that we need to investigate
\begin{eqnarray*}
	\extendedLmatrix{\boot{\para}_{\lceil \rrate^{i} \rceil}}v - \extendedLmatrix{\boot{\para}_{\lceil \rrate^{j} \rceil}}v = \Lmatrix{\boot{\para}_{\lceil \rrate^{i} \rceil}}v_2 - \Lmatrix{\boot{\para}_{\lceil \rrate^{j} \rceil}}v_2,
\end{eqnarray*}
for $j < i$. For this purpose, plug in $\left(\boot{\para}_{\lceil \rrate^{i} \rceil} - \estpara{\lceil \rrate^{i} \rceil}\right) \left[ 0_{r \times p}, I_r \right]'=\bigrandommat_i {X_{22}(i)}^{-1/2}$, as well as \eqref{closed-loop-id}, in Lemma~\ref{optmanifold}. Using the notation of Lemma~\ref{optmanifold}, ${\boot{\para}_{\lceil \rrate^{j} \rceil}}\extendedLmatrix{\boot{\para}_{\lceil \rrate^{j} \rceil}}$ is stable, and for the matrix $M+N \Lmatrix{\boot{\para}_{\lceil \rrate^{j} \rceil}}$, \eqref{closed-loop-id} implies that $\limsup\limits_{j < i; \:\: j \to \infty} i^{-1/2} \rrate^{j/2} \Mnorm{ \left( \boot{\para}_{\lceil \rrate^{i} \rceil} - \boot{\para}_{\lceil \rrate^{j} \rceil} \right) \extendedLmatrix{\boot{\para}_{\lceil \rrate^{j} \rceil}} }{2} < \infty$. Hence
\begin{eqnarray*}
\liminf\limits_{i>j; j \to \infty} \frac{ \norm{\left(\Lmatrix{\boot{\para}_{\lceil \rrate^{i} \rceil}} - \Lmatrix{\boot{\para}_{\lceil \rrate^{j} \rceil}}\right)v_2 }{}}{ \norm{ \bigrandommat_i {X_{22}(i)}^{-1/2}v_2 }{} - i^{1/2} \rrate^{-j/2} \norm{v_2}{2} } > 0.
\end{eqnarray*}

By leveraging \eqref{MCLT}, \eqref{sufficientmineig}, and \eqref{X22vsX12}, we get
\begin{eqnarray} \label{mindiversity}
\liminf\limits_{i \to \infty} \frac{ i^{1/2} \rrate^{-i/2} \eigmin{\boot{U}_{i+1}}^{1/2} }{ \eigmin{X_{22}(i)}^{-1/2} - i \rrate^{-i/2} } > 0.
\end{eqnarray}

Note that since the system is stabilized, $\eigmin{Q_u}>0$ implies that $\sup\limits_{1 \leq k } \Mnorm{\Lmatrix{\boot{\para}_{\lceil \rrate^{k} \rceil}}}{2}<\infty$. Further, according to Lemma~\ref{smallcov} we have $\limsup\limits_{i \to \infty} \Mnorm{X_{12}(i)}{2}<\infty$. Thus, using \eqref{blockeigmineq}, as long as $\limsup\limits_{i \to \infty} i\rrate^{-i} \eigmin{\boot{U}_i}=0$, we have 
\begin{eqnarray*}
	0 < \liminf\limits_{i \to \infty} \frac{\eigmin{\boot{U}_{i}} }{ \eigmin{X_{22}(i)} } \leq \limsup\limits_{i \to \infty} \frac{\eigmin{\boot{U}_{i}} }{ \eigmin{X_{22}(i)} } \leq 1.
\end{eqnarray*}
Therefore, \eqref{mindiversity} yields
\begin{eqnarray} \label{booteigmin}
0< \liminf\limits_{i \to \infty} i^{1/2} \rrate^{-i/2} \eigmin{\boot{U}_{i}} , \:\:\:\:\:\: \limsup\limits_{i \to \infty} \rrate^{-i} \eigmax{\boot{U}_{i}} < \infty.
\end{eqnarray}
In addition, since the extended feedback matrices $\extendedLmatrix{\boot{\para}_{\lceil \rrate^{j} \rceil}}$ in the definitions of $U_i$ and $\boot{U}_i$ are the same, we get a similar result for the actual system:
\begin{eqnarray} \label{eigmin}
0<\liminf\limits_{i \to \infty} i^{1/2} \rrate^{-i/2} \eigmin{{U}_{i}} , \:\:\:\:\:\: \limsup\limits_{i \to \infty} \rrate^{-i} \eigmax{{U}_{i}} < \infty.
\end{eqnarray}
Next, applying Lemma~\ref{Lailemma} to the least squares estimates in \eqref{LSE}, \eqref{BootLSE}, and using \eqref{booteigmin}, \eqref{eigmin}, we obtain
\begin{eqnarray*}
	\limsup\limits_{i \to \infty} i^{-1} \rrate^{i/4} \Mnorm{ \estpara{\lceil \rrate^{i} \rceil} - \para_0 }{2}  &<& \infty, \\
	\limsup\limits_{i \to \infty} i^{-1} \rrate^{i/4} \Mnorm{ \boot{\para}_{\lceil \rrate^{i} \rceil} - \estpara{\lceil \rrate^{i} \rceil} }{2} &<& \infty.
\end{eqnarray*}
Combining the above two inequalities, we get the desired result regarding the learning accuracy. Further, Lemma~\ref{Lipschitzlemma} implies that
\begin{eqnarray*}
	\limsup\limits_{i \to \infty} i^{-2} \rrate^{i/2} \Mnorm{ \Lmatrix{\boot{\para}_{\lceil \rrate^{i} \rceil}} - \Lmatrix{\para_0} }{2}^2 &<& \infty.
\end{eqnarray*}
So, according to Lemma~\ref{smallcov} we get
\begin{eqnarray} \label{regrethistory}
	\limsup\limits_{i \to \infty} i^{-2}\rrate^{-i/2} \sum\limits_{t=0}^{\lceil \rrate^{i} \rceil-1} \norm{ \left( \Lmatrix{\boot{\para}_t} - \Lmatrix{\para_0} \right) \state{t} }{}^2 < \infty .
\end{eqnarray}

On the other hand, the moment condition $\sup\limits_{t \geq 1} \E{\norm{\noise{t}}{2}^{\tailexp}}<\infty$ implies that for all $\tailcoeff> 1/\tailexp$, $\epsilon>0$, 
\begin{eqnarray*}
	\sum\limits_{t=1}^{\infty} \PP{ \norm{\noise{t}}{2} > \epsilon t^\tailcoeff } \leq \sum\limits_{t=1}^{\infty} \epsilon^{-\tailexp} t^{-\tailexp \tailcoeff} \E{\norm{\noise{t}}{2}^\tailexp} < \infty, 
\end{eqnarray*}
which by the Borel-Cantelli Lemma leads to $\limsup\limits_{t \to \infty} t^{-1/4} \norm{\noise{t}}{2}=0$, since $\tailexp>4$. By stability of the closed-loop matrix $A_0+B_0\Lmatrix{\para_0}$, it holds that
\begin{eqnarray} \label{regretfluct}
	\limsup\limits_{n \to \infty} n^{-1/2}{\norm{\sum\limits_{t=1}^n \left(A_0+B_0\Lmatrix{\para_0}\right)^{n-t}\noise{t}}{2}^2} \leq \limsup\limits_{n \to \infty} \left( \sum\limits_{t=1}^n \Mnorm{ \left(A_0+B_0\Lmatrix{\para_0}\right)^{n-t}}{2} t^{-1/4} \norm{\noise{t}}{2} \right)^2  = 0.
\end{eqnarray}
Putting \eqref{regrethistory}, \eqref{regretfluct}, and Lemma~\ref{generalregret} together, the desired result on the growth rate of the regret is established, which finishes the proof.


\fi 
\end{document}